\documentclass[10pt,conference]{IEEEtran}

\ifCLASSINFOpdf
  \usepackage[pdftex]{graphicx}
  \graphicspath{{../}}
\else
\fi

\usepackage{xcolor}

\usepackage{amsmath,amssymb,amsthm,mathtools}
\usepackage{algorithm}
\usepackage{algpseudocode}
\usepackage{enumitem}
\usepackage{booktabs}
\usepackage{bbm}

\usepackage{subcaption} 

\newtheorem{theorem}{Theorem}
\newtheorem{remark}{Remark}
\newtheorem{assumption}{Assumption}
\newtheorem{lemma}{Lemma}
\newtheorem{corollary}{Corollary}

\usepackage{scalerel}

\AtBeginDocument{%
  \addtolength{\topmargin}{0.025in}%
  \addtolength{\textheight}{-0.050in}%
}

\allowdisplaybreaks

\DeclareMathOperator*{\argmin}{arg\;min}

\newcommand{\bseq}{\begin{subequations}}
	\newcommand{\eseq}{\end{subequations}}
\newcommand{\baln}{\begin{align}}
	\newcommand{\ealn}{\end{align}}
\newcommand{\balnd}{\begin{aligned}}
	\newcommand{\ealnd}{\end{aligned}}
\newcommand{\beq}{\begin{equation}}
	\newcommand{\eeq}{\end{equation}}
\newcommand{\beqn}{\begin{eqnarray}}
	\newcommand{\eeqn}{\end{eqnarray}}
\newcommand{\beqno}{\begin{eqnarray*}}
	\newcommand{\eeqno}{\end{eqnarray*}}
\newcommand{\bma}{\begin{displaymath}}
	\newcommand{\ema}{\end{displaymath}}
\newcommand{\bnu}{\begin{enumerate}}
	\newcommand{\enu}{\end{enumerate}}
\newcommand{\bce}{\begin{center}}
	\newcommand{\ece}{\end{center}}
\newcommand{\btb}{\begin{tabular}}
	\newcommand{\etb}{\end{tabular}}
\newcommand{\ba}{\begin{array}}
	\newcommand{\ea}{\end{array}}

\begin{document}

\title{Resilient Decentralized Wireless Federated Learning via Gradient Tracking with AdamW}

%
%
\author{
    \IEEEauthorblockN{Nguyen Van Thieu, Ti Ti Nguyen, Ons Aouedi, Vu Nguyen Ha, and Symeon Chatzinotas}
    \IEEEauthorblockA{\textit{Interdisciplinary Centre for Security, Reliability and Trust (SnT), University of Luxembourg, Luxembourg} \\
    \{vanthieu.nguyen, titi.nguyen, ons.aouedi, vu-nguyen.ha, symeon.chatzinotas\}@uni.lu}
}


%

\markboth{Journal of \LaTeX\ Class Files,~Vol.~14, No.~8, August~2015}%
{Shell \MakeLowercase{\textit{et al.}}: Bare Demo of IEEEtran.cls for IEEE Journals}
%



\maketitle

\begingroup
\renewcommand\thefootnote{}
\footnotetext{%
\copyright~2026 IEEE. Personal use of this material is permitted.
Permission from IEEE must be obtained for all other uses, in any
current or future media, including reprinting/republishing this
material for advertising or promotional purposes, creating new
collective works, for resale or redistribution to servers or lists,
or reuse of any copyrighted component of this work in other works.%
}
\addtocounter{footnote}{-1}
\endgroup

\begin{abstract}

Wireless Internet-of-Things (IoT) edge networks require decentralized learning (DecL) methods that can operate reliably under both heterogeneous local data and communication-constrained wireless links. However, existing decentralized optimization schemes often incur substantial communication overhead and degraded performance when transmissions are constrained by strict airtime budgets, fading channels, and packet losses. This paper proposes QEF-GT-AdamW, a communication-efficient and outage-resilient algorithm for DecL over wireless communication (WCom) networks. The proposed method combines gradient tracking to mitigate the effect of non-IID data, AdamW-based adaptive optimization to improve training stability, and dual-stream biased quantization with error feedback to reduce communication payloads for both model and tracking exchanges. 
To address unreliable broadcast communication, the proposed framework further employs a local fallback strategy when scheduled packets are not successfully received. 
We explicitly model the effect of bandwidth, transmit power, airtime constraints, and fading channels on DecL performance, and establish convergence guarantees for the proposed algorithm under compressed and unreliable wireless communication.
Experimental results on heterogeneous MNIST and CIFAR-10 settings show that QEF-GT-AdamW consistently improves robustness and convergence performance over representative DecL baselines while achieving favorable accuracy-communication trade-offs under limited wireless resources.
\end{abstract}



%
\IEEEpeerreviewmaketitle

\section{Introduction}
\label{sec:introduction}

IoT wireless communication (WCom) networks increasingly rely on decentralized learning (DecL) to train models directly over distributed devices without a central coordinator \cite{beltran2023decentralized}. While this paradigm improves scalability and robustness, its practical deployment is challenged by non-IID local data and unreliable WCom arising from limited bandwidth, strict airtime budgets, fading channels, and packet losses. These issues are tightly coupled in wireless-based DecL, where WCom constraints directly affect learning performance \cite{chen2024communication}.

Among DecL optimization methods, gradient tracking (GT) is an effective mechanism for mitigating client drift under heterogeneous data \cite{pu2021distributed}, since it enables each node to asymptotically follow the network-wide gradient. However, classical GT requires repeated exchanges of both model and tracking variables, which leads to substantial WCom overhead. Recent work such as GTAdam \cite{carnevale2022gtadam} combines GT with Adam-type adaptive momentum estimation to improve convergence speed and robustness, but existing GT-Adam methods are mainly developed under ideal communication assumptions. In parallel, gossip-based SGD \cite{koloskova2019decentralized} and compressed DecL methods \cite{chen2024communication, richtarik2021ef21} reduce payload sizes through limited information exchange and quantization, yet they typically either lack the heterogeneity-handling capability of GT or do not explicitly address the interaction among adaptive optimization, compression, and unreliable wireless links. However, a practically deployable framework for wireless DecL under non-IID data and constrained communication remains missing.

Motivated by this gap, we develop QEF-GT-AdamW, a communication-efficient and outage-resilient DecL framework for wireless edge networks. The proposed approach combines GT for heterogeneity mitigation, AdamW-based adaptive updates, and dual-stream quantization with error feedback (EF) to reduce the payloads associated with exchanging both model and tracking variables, while a local fallback mechanism improves robustness when neighbor packets are not successfully received. In contrast to idealized DecL models, the framework explicitly models synchronized wireless learning with computation-dependent airtime, bandwidth and transmit-power constraints, and fading-induced outages. 
We also establish convergence guarantees for smooth convex objectives under compression and unreliable communication, together with a stronger linear-rate characterization under an additional gradient-dominance condition.
Experiments on heterogeneous MNIST and CIFAR-10 settings over fading wireless channels show that the proposed method achieves favorable robustness-accuracy-communication trade-offs compared with representative decentralized baselines.

\section{System Model}
\label{sec:system_model}

We consider a DecL framework over a WCom network with $N$ nodes, where each node holds a local dataset and exchanges information only with its neighbors over WCom channels. The goal is to minimize the overall average empirical loss under non-IID data distributions.

\subsection{DecL Background}

\subsubsection{Decentralized GT Learning}

GT enables nodes to asymptotically follow the global gradient, even under heterogeneous local data distributions.
Each node $i$ maintains a local model $\mathbf{x}_i^k \in \mathbb{R}^d$ and an auxiliary GT variable $\mathbf{y}_i^k \in \mathbb{R}^d$.
Let $f_i:\mathbb{R}^d\rightarrow\mathbb{R}$ be the local loss function at node $i$, and let
$\mathbf{W}=[w_{ij}]\in\mathbb{R}^{N\times N}$ be a fixed doubly-stochastic baseline mixing matrix satisfying
$\sum_{j=1}^{N}w_{ij}=1$,
$\sum_{i=1}^{N}w_{ij}=1$, and
$w_{ij}\geq0$.
The classical GT updates \cite{pu2021distributed} at iteration $k$ are:
$\mathbf{x}_i^{k+1} = \sum_{j=1}^{N} w_{ij}\mathbf{x}_j^k -\alpha\mathbf{y}_i^k,$ and 
$\mathbf{y}_i^{k+1} = \sum_{j=1}^{N} w_{ij}\mathbf{y}_j^k
+\nabla f_i(\mathbf{x}_i^{k+1}) -\nabla f_i(\mathbf{x}_i^k),$
where $\alpha>0$ is a step size.
This mechanism enables the tracking variables to follow the network-wide gradient and supports convergence toward the global objective
\begin{equation}
F(\mathbf{x})
\triangleq
\frac{1}{N}
\sum_{i=1}^{N} f_i(\mathbf{x}).
\label{eq:global_obj}
\end{equation}

\subsubsection{Adaptive Optimization and AdamW}

Adam is an adaptive optimization algorithm that uses exponential moving averages of first- and second-order gradients to accelerate convergence. However, its coupling of weight decay with adaptive updates can lead to inconsistent regularization and reduced generalization, especially in decentralized settings. AdamW \cite{zhou2024towards} addresses this by \emph{decoupling weight decay from gradient-based updates}, leading to more stable optimization.

\subsection{DecL with Wireless Communication}
\label{subsec:wireless}

\begin{figure}[t]
    \centering
    \includegraphics[width=0.95\linewidth]{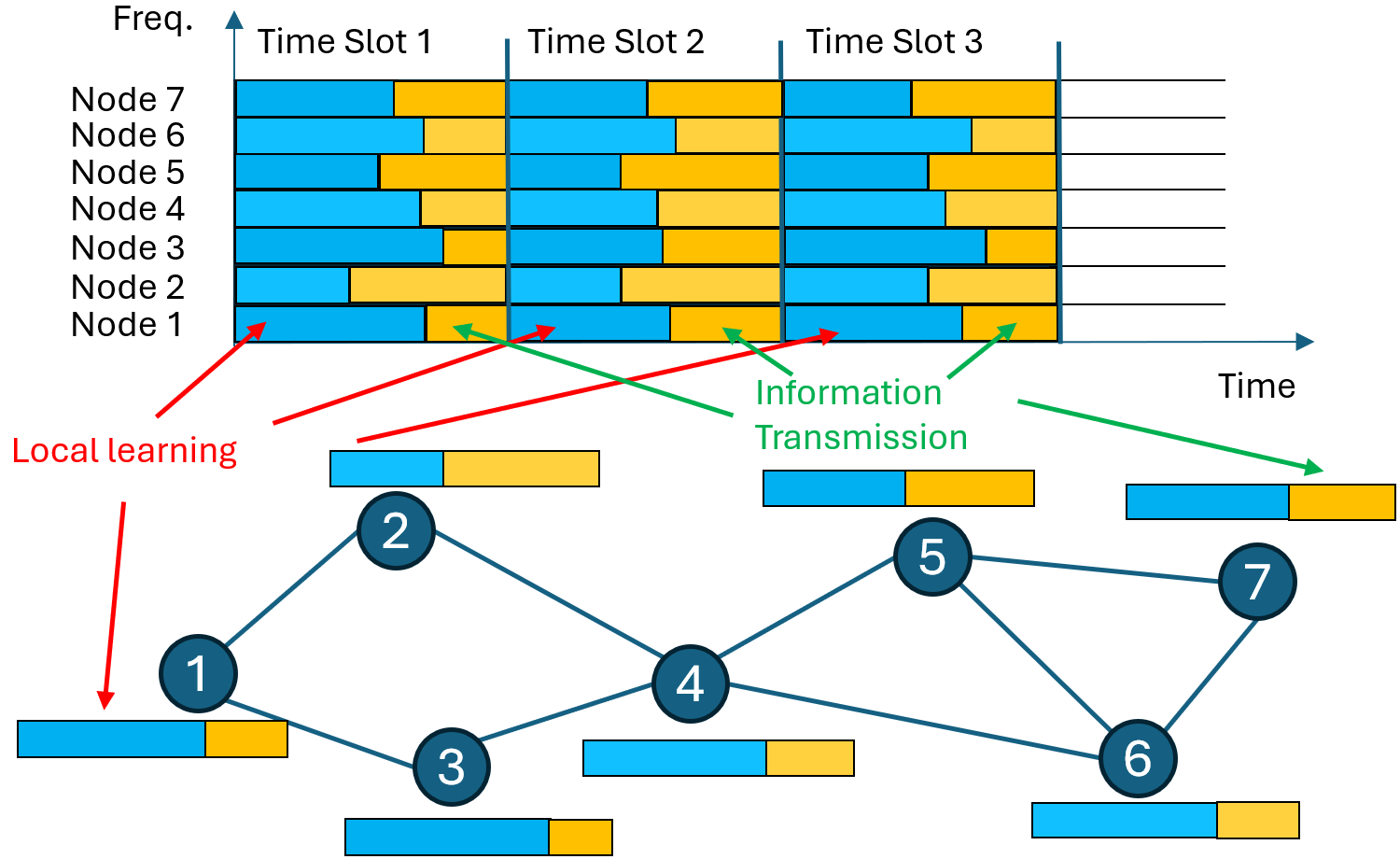}
    \captionsetup{font=footnotesize}
    \caption{A TS-based learning and information exchange framework.}
    \label{fig:sysmodel}
\end{figure}

\subsubsection{Time-Slot based Learning and Information Exchange Framework}

A synchronized time-slot (TS) protocol governs training and communication, as illustrated in Fig.~\ref{fig:sysmodel}. In each TS, nodes perform local updates and transmit over orthogonal subcarriers with limited airtime and unreliable links. Each TS has a strict deadline $\Delta t_{\max}^{(\mathrm{iter})}$.
Node $i$'s local computation time is modeled as
$T_{\mathrm{comp},i}^k = R_i^k\Delta T,$
where $R_i^k$ follows a truncated normal distribution \cite{burkardt2014truncated} and $\Delta T>0$ denotes the basic computation-time unit.
The remaining airtime available to node $i$ for transmission is
\begin{equation}
\Delta t_{\mathrm{air},i}^k
=
\Delta t_{\max}^{(\mathrm{iter})}
-
T_{\mathrm{comp},i}^k.
\end{equation}
This implies that devices with slower computation have less time to transmit, thus increasing the risk of outage.

\subsubsection{Channel Model and Outage Probability}

Let $d_{ij}$ denote the Euclidean distance between nodes $i$ and $j$.
The pathloss is given by
$\nu_{ij} = \nu_0\left(d_0/d_{ij}\right)^{\zeta},$
where $\zeta>0$ is the pathloss exponent, $\nu_0$ is the reference pathloss at distance $d_0$, and $d_0$ is the reference distance.
The small-scale fading coefficient $h_{ij}^k$ follows a Rician distribution with Rician factor $\kappa\geq0$.

Let $P_i$ and $B_i$ denote the transmission power and bandwidth of node $i$, respectively.
The average SNR of the transmission from node $i$ to node $j$ is
$\bar{\gamma}_{ij}=\frac{P_i\nu_{ij}}{N_0B_i},$
where $N_0$ is the noise power spectral density.
The outage probability at iteration $k$ is estimated as \cite{simon2004digital}
\begin{equation}
\label{eq:outage-rician}
p_{ij}^{\sf outage,k}
=
1-
\mathcal{Q}_1^{\sf M}
\left(
\sqrt{2\kappa},
\sqrt{
\frac{
2(1+\kappa)\gamma_i^{\sf th,k}
}{
\bar{\gamma}_{ij}
}
}
\right),
\end{equation}
where $\mathcal{Q}_1^{\sf M}(\cdot,\cdot)$ is the first-order Marcum $Q$-function \cite{simon2004digital}, and $\gamma_i^{\sf th,k}$ is the SNR threshold of node $i$ at iteration $k$.
The threshold $\gamma_i^{\sf th,k}$ is determined by the amount of information exchanged and the available airtime according to
$B_i\log_2\left(1+\gamma_i^{\sf th,k}\right)
=D_i/\Delta t_{\mathrm{air},i}^k,$
where $D_i$ denotes the amount of information transmitted by node $i$, which depends on the model dimension $d$ and the compression scheme.

\section{Proposed Algorithm: QEF-GT-AdamW}
\label{sec:algorithm}

This section presents QEF-GT-AdamW, which combines robust consensus mixing under packet losses, AdamW-based adaptive updates, GT for heterogeneous data, and dual-stream quantization with EF for both model and tracking exchanges.

\subsection{Initialization and Robust Mixing}
\label{subsec:init_mixing}

Each node $i$ maintains a local model $\mathbf{x}_i^k\in\mathbb{R}^d$, a gradient-tracking variable $\mathbf{y}_i^k\in\mathbb{R}^d$, AdamW moment buffers $(\mathbf{m}_i^k,\mathbf{v}_i^k)$, and dual EF residuals $(\mathbf{e}_{x,i}^k,\mathbf{e}_{y,i}^k)$.
Let $\mathcal{N}_i$ denote the neighbor set of node $i$ in the baseline communication graph, and let $\mathcal{N}_i^{\mathrm{in}}$ denote its set of potential in-neighbors.
To prevent duplicate processing, node $i$ stores the last accepted sequence number $\mathrm{LastSeq}_i[j]$ for each $j\in\mathcal{N}_i^{\mathrm{in}}$.

To ensure that both the primal and tracking streams are initialized consistently before the first communication round, the warm start is performed as in Algorithm~\ref{alg:init_algo}.
Starting from an initial model $\mathbf{x}_i^0$, node $i$ computes its local gradient
$\mathbf{g}_i^0=\nabla f_i(\mathbf{x}_i^0)$ and sets
$\mathbf{y}_i^0=\mathbf{g}_i^0$.
The AdamW moment buffers are initialized to zero.
For consistency with the subsequent EF recursion, the pre-initialization residuals are set to
$\mathbf{e}_{x,i}^{-1}=\mathbf{0}$ and
$\mathbf{e}_{y,i}^{-1}=\mathbf{0}$.
The initial model and tracking variables are then compressed using the corresponding compressors $\mathcal{Q}_x$ and $\mathcal{Q}_y$ before being broadcast to neighboring nodes.

\begin{algorithm}[t]
\caption{\small \textsc{Initialization at node $i$ (Warm Start)}}
\label{alg:init_algo}
\begin{algorithmic}[1]
\State \textbf{Given:} mixing weights $\{w_{ij}\}$; stepsizes $\{\alpha_k\}$; AdamW parameters $(\beta_m,\beta_v,\varepsilon,\lambda)$; compressors $\mathcal{Q}_x,\mathcal{Q}_y$.
\State Initialize $\mathbf{x}_i^0$.
\State Compute $\mathbf{g}_i^0=\nabla f_i(\mathbf{x}_i^0)$ and set $\mathbf{y}_i^0=\mathbf{g}_i^0$.
\State Set $\mathbf{m}_i^0=\mathbf{0}$ and $\mathbf{v}_i^0=\mathbf{0}$.
\State Set $\mathbf{e}_{x,i}^{-1}=\mathbf{0}$ and $\mathbf{e}_{y,i}^{-1}=\mathbf{0}$.
\State Set $\mathrm{Seq}_i\gets0$ and $\mathrm{LastSeq}_i[j]\gets-1$ for all $j\in\mathcal{N}_i^{\mathrm{in}}$.
\State Warm-start compression:
$\tilde{\mathbf{x}}_i^0
\gets
\mathbf{x}_i^0+\mathbf{e}_{x,i}^{-1}$,
$\mathbf{c}_{x,i}^0
\gets
\mathcal{Q}_x(\tilde{\mathbf{x}}_i^0)$,
$\mathbf{e}_{x,i}^0
\gets
\tilde{\mathbf{x}}_i^0-\mathbf{c}_{x,i}^0$.
\State Similarly,
$\tilde{\mathbf{y}}_i^0
\gets
\mathbf{y}_i^0+\mathbf{e}_{y,i}^{-1}$,
$\mathbf{c}_{y,i}^0
\gets
\mathcal{Q}_y(\tilde{\mathbf{y}}_i^0)$,
$\mathbf{e}_{y,i}^0
\gets
\tilde{\mathbf{y}}_i^0-\mathbf{c}_{y,i}^0$.
\State Broadcast packet $(\mathrm{Seq}_i,\mathbf{c}_{x,i}^0,\mathbf{c}_{y,i}^0)$ to $\mathcal{N}_i$.
\end{algorithmic}
\end{algorithm}

At iteration $k$, node $i$ accepts a decoded packet
$(\hat{\mathbf{x}}_j^k,\hat{\mathbf{y}}_j^k)$ from neighbor $j$ only if its sequence number is larger than $\mathrm{LastSeq}_i[j]$; otherwise, the packet is treated as missing.
Here, the decoded vectors correspond to the compressor outputs, i.e.,
$\hat{\mathbf{x}}_j^k=\mathbf{c}_{x,j}^k$ and
$\hat{\mathbf{y}}_j^k=\mathbf{c}_{y,j}^k$.

Let $a_{j\to i}^k\in\{0,1\}$ indicate successful reception, where
$a_{j\to i}^k=1$ means that node $i$ successfully accepts the packet transmitted by node $j$, and $a_{j\to i}^k=0$ otherwise.
Under the outage model in \eqref{eq:outage-rician},
$\Pr(a_{j\to i}^k=1) = 1-p_{ji}^{\sf outage,k}.$
To preserve stochastic mixing under packet losses, missing neighbor weights are reassigned to the self-loop:
\begin{equation}
b_{ij}^k
\triangleq
w_{ij}a_{j\to i}^k,
\quad j\neq i,
\qquad
b_{ii}^k
\triangleq
1-\sum_{j\neq i}b_{ij}^k.
\end{equation}
Thus, the realized mixing matrix
$\mathbf{B}^k=[b_{ij}^k]\in\mathbb{R}^{N\times N}$
remains row-stochastic.
The robust mixed variables used by node $i$ are
$\mathbf{X}_i^k=b_{ii}^k\mathbf{x}_i^k+\sum_{j\neq i}
b_{ij}^k\hat{\mathbf{x}}_j^k,$
and
$\mathbf{Y}_i^k=b_{ii}^k\mathbf{y}_i^k+\sum_{j\neq i}b_{ij}^k\hat{\mathbf{y}}_j^k.$

Define the stacked vectors
\begin{equation}
\begin{aligned}
\mathbf{x}^k
&\triangleq
\operatorname{col}
(\mathbf{x}_1^k,\ldots,\mathbf{x}_N^k),
\quad
\mathbf{y}^k
\triangleq
\operatorname{col}
(\mathbf{y}_1^k,\ldots,\mathbf{y}_N^k),\\
\hat{\mathbf{x}}^k
&\triangleq
\operatorname{col}
(\hat{\mathbf{x}}_1^k,\ldots,\hat{\mathbf{x}}_N^k),
\quad
\hat{\mathbf{y}}^k
\triangleq
\operatorname{col}
(\hat{\mathbf{y}}_1^k,\ldots,\hat{\mathbf{y}}_N^k),\\
\mathbf{X}^k
&\triangleq
\operatorname{col}
(\mathbf{X}_1^k,\ldots,\mathbf{X}_N^k),
\quad
\mathbf{Y}^k
\triangleq
\operatorname{col}
(\mathbf{Y}_1^k,\ldots,\mathbf{Y}_N^k).
\end{aligned}
\end{equation}

Furthermore, define
$\mathbf{A}^k=[A_{ij}^k]\in\mathbb{R}^{N\times N}$ by
$A_{ij}^k = b_{ij}^k$ for $j \neq i$ and $0$ if $j = i$.
Then, the robust mixing operation admits the compact form
\begin{equation}
\mathbf{X}^k
=
(\mathbf{B}^k\otimes\mathbf{I}_d)\mathbf{x}^k
+
(\mathbf{A}^k\otimes\mathbf{I}_d)
(\hat{\mathbf{x}}^k-\mathbf{x}^k),
\end{equation}
with the analogous expression for $\mathbf{Y}^k$.
These mixed variables serve as the robust consensus states used in the local adaptive update and gradient-tracking steps described next.

\subsection{AdamW Update, GT, and Dual Quantization with EF}
\label{subsec:adamw_gt_qef}

Using the mixed tracking variable $\mathbf{Y}_i^k$ as a robust estimate of the global gradient, node $i$ updates its local moment estimates and applies bias correction according to
\begin{equation}
\begin{aligned}
\mathbf{m}_i^{k+1}
&=
\beta_m\mathbf{m}_i^k
+
(1-\beta_m)\mathbf{Y}_i^k,\\
\mathbf{v}_i^{k+1}
&=
\beta_v\mathbf{v}_i^k
+
(1-\beta_v)
\left(
\mathbf{Y}_i^k\odot\mathbf{Y}_i^k
\right),
\end{aligned}
\label{eq:ud:v_i}
\end{equation}
\begin{equation}
\label{eq:adam_bias}
\hat{\mathbf{m}}_i^{k+1}
=
\mathbf{m}_i^{k+1}/(1-\beta_m^{k+1}),
\quad
\hat{\mathbf{v}}_i^{k+1}
=
\mathbf{v}_i^{k+1}(1-\beta_v^{k+1}),
\end{equation}
where $\odot$ denotes element-wise multiplication, and $\beta_m,\beta_v\in(0,1)$ are the momentum parameters.
The primal model is subsequently updated using the AdamW rule:
\begin{equation}
\mathbf{x}_i^{k+1}
=
\mathbf{X}_i^k
-
\alpha_k
\left(
\frac{
\hat{\mathbf{m}}_i^{k+1}
}{
\sqrt{\hat{\mathbf{v}}_i^{k+1}}
+
\varepsilon
}
+
\lambda\mathbf{X}_i^k
\right),
\label{eq:primal_update}
\end{equation}
where $\alpha_k>0$ is the step size, $\varepsilon>0$ is a small numerical constant, and $\lambda\geq0$ is the decoupled weight-decay coefficient.

After the primal update, node $i$ computes the new local gradient
$\mathbf{g}_i^{k+1}=\nabla f_i(\mathbf{x}_i^{k+1}),$
and updates the tracking variable as
\begin{equation}
\mathbf{y}_i^{k+1}
=
\mathbf{Y}_i^k
+
\left(
\mathbf{g}_i^{k+1}
-
\mathbf{g}_i^k
\right).
\label{eq:tracking_update}
\end{equation}

To reduce WCom overhead, both $\mathbf{x}_i^{k+1}$ and $\mathbf{y}_i^{k+1}$ are compressed using separate compressors $\mathcal{Q}_x$ and $\mathcal{Q}_y$ with EF \cite{karimireddy2019error}.
For $z\in\{x,y\}$,
\begin{align}
\tilde{\mathbf{z}}_i^{k+1}
&=
\mathbf{z}_i^{k+1}
+
\mathbf{e}_{z,i}^k \quad \text{(Add residual)},
\\
\mathbf{c}_{z,i}^{k+1}
&=
\mathcal{Q}_z
\left(
\tilde{\mathbf{z}}_i^{k+1}
\right) \quad \text{(Compress)},
\\
\mathbf{e}_{z,i}^{k+1}
&=
\tilde{\mathbf{z}}_i^{k+1}
-
\mathbf{c}_{z,i}^{k+1} \quad \text{(Update residual)}.
\end{align}

Finally, node $i$ increments its local sequence counter and transmits
$(\mathrm{Seq}_i,\mathbf{c}_{x,i}^{k+1},\mathbf{c}_{y,i}^{k+1})$
to its neighbors.
The overall per-round routine is summarized in Algorithm~\ref{alg:qef_gt_adamw}.

\begin{algorithm}[!t]
\caption{\small \textsc{QEF-GT-AdamW at node $i$ (iteration $k\geq0$)}}
\label{alg:qef_gt_adamw}
\small
\begin{algorithmic}[1]
\State \textbf{Reception (decode new packets only):}
\For{each $j\in\mathcal{N}_i$}
    \If{$a_{j\to i}^k=1$ and received $\mathrm{Seq}_j>\mathrm{LastSeq}_i[j]$}
        \State Decode payload into $(\hat{\mathbf{x}}_j^k,\hat{\mathbf{y}}_j^k)$; set $\mathrm{LastSeq}_i[j]\gets\mathrm{Seq}_j$.
    \Else
        \State Treat $j$ as \textit{missing at iteration $k$}.
    \EndIf
\EndFor
\State \textbf{Robust Mixing with Local Fallback:}
Compute $\mathbf{X}_i^k$ and $\mathbf{Y}_i^k$ using the dynamic weights defined in Section~\ref{subsec:init_mixing}.
\State \textbf{AdamW Update:}
Update $\mathbf{m}_i^{k+1}$ and $\mathbf{v}_i^{k+1}$ using $\mathbf{Y}_i^k$ as in \eqref{eq:ud:v_i}.
Compute the bias-corrected moments $\hat{\mathbf{m}}_i^{k+1}$ and $\hat{\mathbf{v}}_i^{k+1}$.
Update the local model according to \eqref{eq:primal_update}.
\State \textbf{Gradient Tracking Update:}
Compute $\mathbf{g}_i^{k+1}=\nabla f_i(\mathbf{x}_i^{k+1})$ and update the tracking variable according to \eqref{eq:tracking_update}.
\State \textbf{Compression with EF Strategy:}
For each $z\in\{x,y\}$, compute $\tilde{\mathbf{z}}_i^{k+1}$, $\mathbf{c}_{z,i}^{k+1}$, and $\mathbf{e}_{z,i}^{k+1}$ using the EF recursion.
\State \textbf{Broadcast:}
Increment $\mathrm{Seq}_i\gets\mathrm{Seq}_i+1$.
Broadcast packet $(\mathrm{Seq}_i,\mathbf{c}_{x,i}^{k+1},\mathbf{c}_{y,i}^{k+1})$ to $\mathcal{N}_i$.
\State Set $\mathbf{g}_i^k\gets\mathbf{g}_i^{k+1}$.
\end{algorithmic}
\end{algorithm}

\subsection{Convergence Analysis}
\label{sec:convergence}

\noindent\textit{Reweighted Objective and Notation:}
For the convergence analysis, we specialize to a constant stepsize
$\alpha_k\equiv\alpha>0$.
Due to packet-dependent local fallback, $\mathbf{B}^k$ is row-stochastic but generally not doubly stochastic.
Hence, the uniform network average is no longer preserved, and the convergence dynamics are governed by the Perron weighting induced by the effective mixing matrix \cite{nedic2014distributed}.

Let
$\bar{\mathbf{B}}\triangleq\mathbb{E}[\mathbf{B}^k]$
and let $\boldsymbol{\pi}\succ\mathbf{0}$ be the normalized left Perron vector \cite{pu2020push} satisfying
$\boldsymbol{\pi}^{\top}\bar{\mathbf{B}}=\boldsymbol{\pi}^{\top},$ and
$\boldsymbol{\pi}^{\top}\mathbf{1}=1.$
Accordingly, we analyze convergence with respect to the reweighted objective
\begin{equation}
F_{\pi}(\mathbf{x})
\triangleq
\sum_{i=1}^{N}\pi_i f_i(\mathbf{x}),
\qquad
\mathbf{x}_{\pi}^{\star}
\in
\argmin_{\mathbf{x}\in\mathbb{R}^d}
F_{\pi}(\mathbf{x}),
\label{eq:perron-objective}
\end{equation}
and define
$F_{\pi}^{\star} \triangleq F_{\pi}(\mathbf{x}_{\pi}^{\star}).$
The Perron weights characterize the long-term asymmetry induced by packet-dependent mixing.
In the doubly-stochastic case,
$\boldsymbol{\pi}=\frac{1}{N}\mathbf{1}$
and $F_{\pi}$ reduces to the original objective \eqref{eq:global_obj}.
For a stacked vector
$\mathbf{z}=\operatorname{col}(\mathbf{z}_1,\ldots,\mathbf{z}_N)\in\mathbb{R}^{Nd}$,
define
\begin{equation}
\bar{\mathbf{z}}_{\pi}
\triangleq
(\boldsymbol{\pi}^{\top}\otimes\mathbf{I}_d)\mathbf{z},
\qquad
\mathbf{\Pi}_{\pi}^{\perp}
\triangleq
\mathbf{I}_{Nd}
-
(\mathbf{1}\boldsymbol{\pi}^{\top}\otimes\mathbf{I}_d).
\label{eq:perron-notation}
\end{equation}
Thus,
$\|\mathbf{\Pi}_{\pi}^{\perp}\mathbf{z}\|$
measures model disagreement from Perron consensus.
Define
\begin{equation}
\Xi^k
\triangleq
\frac{1}{N}
\mathbb{E}
\left[
\|
\mathbf{\Pi}_{\pi}^{\perp}\mathbf{x}^k
\|^2
\right],
\qquad
\bar{\mathbf{g}}_{\pi}^k
\triangleq
\sum_{i=1}^{N}
\pi_i\nabla f_i(\mathbf{x}_i^k).
\label{eq:Xi-gbar}
\end{equation}

\begin{assumption}[Objective regularity \cite{koloskova2021improved}]
\label{ass:objective}
Each $f_i$ is convex and $L$-smooth, and $F_{\pi}$ admits at least one minimizer.
\end{assumption}

\begin{assumption}[Compression \cite{richtarik2021ef21}]
\label{ass:compression}
For $z\in\{x,y\}$, the compressor
$\mathcal{Q}_z:\mathbb{R}^d\rightarrow\mathbb{R}^d$
satisfies
\begin{equation}
\mathbb{E}
\|
\mathcal{Q}_z(\mathbf{v})
-
\mathbf{v}
\|^2
\leq
(1-\delta_z)
\|\mathbf{v}\|^2,
\qquad
0<\delta_z\leq1,
\label{eq:compression}
\end{equation}
for all $\mathbf{v}\in\mathbb{R}^d$.
For deterministic Top-$K$, $\delta_z=K_z/d$, where $K_z$ denotes the number of retained coordinates for stream $z$.
\end{assumption}

\begin{assumption}[Random mixing \cite{pu2020push}]
\label{ass:mixing}
The matrices $\{\mathbf{B}^k\}$ are i.i.d.\ and row-stochastic,
$\bar{\mathbf{B}}$ is primitive, and there exists
$\rho_{\pi}\in(0,1]$ such that
\begin{equation}
\mathbb{E}
\left[
\left\|
\mathbf{\Pi}_{\pi}^{\perp}
(\mathbf{B}^k\otimes\mathbf{I}_d)
\mathbf{z}
\right\|^2
\right]
\leq
(1-\rho_{\pi})
\left\|
\mathbf{\Pi}_{\pi}^{\perp}\mathbf{z}
\right\|^2.
\end{equation}
The current packet-delivery randomness is independent of the algorithmic states and compression variables generated before the communication phase of round $k$.
\end{assumption}

\begin{assumption}[Bounded states \cite{beznosikov2022distributed}]
\label{ass:bounded}
There exist finite constants $R_x,R_y,G_{\infty}>0$ such that
\begin{equation}
\frac{1}{N}
\sum_{i=1}^{N}
\mathbb{E}
\|\mathbf{z}_i^k\|^2
\leq
R_z^2,
\quad
z\in\{x,y\},
\quad
\|\mathbf{Y}_i^k\|_{\infty}
\leq
G_{\infty}.
\label{eq:bounded}
\end{equation}
\end{assumption}

For AdamW preconditioner
$\mathbf{D}_i^{k+1}
\triangleq
\operatorname{diag}
\left(
\frac{1}
{\sqrt{\hat{\mathbf{v}}_i^{k+1}}+\varepsilon}
\right),$
where the operations inside the diagonal are element-wise,
Assumption~\ref{ass:bounded} implies
\begin{equation}
h_{\min}\mathbf{I}_d
\preceq
\mathbf{D}_i^{k+1}
\preceq
h_{\max}\mathbf{I}_d,
\quad
h_{\min}=\frac{1}{G_{\infty}+\varepsilon},
\quad
h_{\max}=\frac{1}{\varepsilon}.
\label{eq:D-bound}
\end{equation}

For $z\in\{x,y\}$, the EF recursion gives
\begin{equation}
\hat{\mathbf{z}}_i^k
=
\mathcal{Q}_z
\left(
\mathbf{z}_i^k
+
\mathbf{e}_{z,i}^{k-1}
\right),
\qquad
\mathbf{e}_{z,i}^k
=
\mathbf{z}_i^k
+
\mathbf{e}_{z,i}^{k-1}
-
\hat{\mathbf{z}}_i^k.
\label{eq:EF}
\end{equation}
Here,
$\hat{\mathbf{z}}_i^k\equiv\mathbf{c}_{z,i}^k$
denotes the decoded compressor output.

\begin{lemma}[Communication-error bounds]
\label{lem:errors}
Under Assumptions~\ref{ass:compression}--\ref{ass:bounded},
\begin{equation}
\frac{1}{N}
\sum_{i=1}^{N}
\mathbb{E}
\|\mathbf{e}_{z,i}^k\|^2
\leq
C_{e,z}R_z^2,
\quad
C_{e,z}
=
\frac{
2(1-\delta_z)(2-\delta_z)
}{
\delta_z^2
},
\label{eq:EF-bound}
\end{equation}
for $0<\delta_z<1$, while $C_{e,z}=0$ for $\delta_z=1$.
Defining
\begin{equation}
\boldsymbol{\chi}_{z,i}^k
\triangleq
\hat{\mathbf{z}}_i^k-\mathbf{z}_i^k,
\qquad
\mathcal{C}_z^k
\triangleq
\frac{1}{N}
\sum_{i=1}^{N}
\mathbb{E}
\|\boldsymbol{\chi}_{z,i}^k\|^2.
\end{equation}
Then,
\begin{equation}
\mathcal{C}_z^k
\leq
4C_{e,z}R_z^2.
\label{eq:C-bound}
\end{equation}

Furthermore, define
$\boldsymbol{\chi}_z^k
\triangleq
\operatorname{col}
\left(
\boldsymbol{\chi}_{z,1}^k,
\ldots,
\boldsymbol{\chi}_{z,N}^k
\right).$
Let
$\mathbf{w}^k
\triangleq
(\boldsymbol{\pi}^{\top}\otimes\mathbf{I}_d)
\mathbf{X}^k$ and
$\mathbf{r}_x^k
\triangleq
\mathbf{w}^k-\bar{\mathbf{x}}_{\pi}^k$.
Then,
\begin{equation}
\mathcal{R}_x^k
\triangleq
\mathbb{E}
\|\mathbf{r}_x^k\|^2
\leq
2N\kappa_B\Xi^k
+
2N\kappa_A\mathcal{C}_x^k,
\label{eq:Rx-bound}
\end{equation}
where
$\kappa_B \triangleq \mathbb{E}\left\|
\boldsymbol{\pi}^{\top}(\mathbf{B}^k-\bar{\mathbf{B}})
\right\|_2^2,$ and
$\kappa_A\triangleq \mathbb{E}\left\| \boldsymbol{\pi}^{\top}\mathbf{A}^k \right\|_2^2.$
\end{lemma}

\begin{proof}
Using \eqref{eq:compression}, Young's inequality with
$\tau_z=\frac{2(1-\delta_z)}{\delta_z}$ gives
\begin{equation}
\mathbb{E}
\|\mathbf{e}_{z,i}^k\|^2
\leq
\left(
1-\frac{\delta_z}{2}
\right)
\mathbb{E}
\|\mathbf{e}_{z,i}^{k-1}\|^2
+
\frac{
(1-\delta_z)(2-\delta_z)
}{
\delta_z
}
\mathbb{E}
\|\mathbf{z}_i^k\|^2.
\end{equation}
Averaging over all nodes and unrolling the recursion gives \eqref{eq:EF-bound}.
Since
$\boldsymbol{\chi}_{z,i}^k = \mathbf{e}_{z,i}^{k-1} - \mathbf{e}_{z,i}^k,$
\eqref{eq:C-bound} follows.

Furthermore, from
$\mathbf{X}^k=(\mathbf{B}^k\otimes\mathbf{I}_d)\mathbf{x}^k +(\mathbf{A}^k\otimes\mathbf{I}_d) \boldsymbol{\chi}_x^k$
and
$(\mathbf{B}^k-\bar{\mathbf{B}})\mathbf{1}= \mathbf{0},$
we obtain
\begin{equation}
\mathbf{r}_x^k
=\left(\boldsymbol{\pi}^{\top}
(\mathbf{B}^k-\bar{\mathbf{B}})\otimes\mathbf{I}_d\right)
\mathbf{\Pi}_{\pi}^{\perp}\mathbf{x}^k
+
\left(\boldsymbol{\pi}^{\top}\mathbf{A}^k\otimes\mathbf{I}_d\right)
\boldsymbol{\chi}_x^k.
\end{equation}
Independence and
$\|\mathbf{a}+\mathbf{b}\|^2
\leq
2\|\mathbf{a}\|^2
+
2\|\mathbf{b}\|^2$
then give \eqref{eq:Rx-bound}.
\end{proof}

We next express the Perron-average algorithm as an inexact gradient method.
Define 
$\bar{\mathbf{Y}}_{\pi}^k
\triangleq
\sum_{i=1}^{N}
\pi_i\mathbf{Y}_i^k,$ 
$\bar{\mathbf{g}}_{\pi}^k
\triangleq
\sum_{i=1}^{N}
\pi_i\nabla f_i(\mathbf{x}_i^k),$
$
\mathbf{t}^k
\triangleq
\bar{\mathbf{Y}}_{\pi}^k
-
\bar{\mathbf{g}}_{\pi}^k,$
and 
$\mathcal{T}^k
\triangleq
\mathbb{E}
\|\mathbf{t}^k\|^2.$
The tracking-stream compression is included in $\mathcal{T}^k$ because $\mathbf{Y}_i^k$ is the received post-mixing tracking state.
Let
\begin{equation}
\mathbf{a}^k
\triangleq
\sum_{i=1}^{N}
\pi_i
\left(
\mathbf{D}_i^{k+1}
\hat{\mathbf{m}}_i^{k+1}
-
\mathbf{Y}_i^k
\right),
\quad
\mathcal{A}^k
\triangleq
\mathbb{E}
\|\mathbf{a}^k\|^2.
\label{eq:A}
\end{equation}
With
$\mathcal{M}^k\triangleq \sum_{i=1}^{N}\pi_i
\mathbb{E}\left\|\hat{\mathbf{m}}_i^{k+1} - \mathbf{Y}_i^k \right\|^2,$
bounded preconditioning gives
\begin{equation}
\mathcal{A}^k
\leq
2h_{\max}^2\mathcal{M}^k
+
2h_{\Delta}^2dG_{\infty}^2,
\qquad
\mathcal{M}^k
\leq
4dG_{\infty}^2,
\label{eq:A-bound}
\end{equation}
where
$h_{\Delta} \triangleq \max \left\{|h_{\min}-1|, |h_{\max}-1| \right\}.$

By $L$-smoothness,
\begin{equation}
\mathbb{E}
\left\|
\bar{\mathbf{g}}_{\pi}^k
-
\nabla F_{\pi}
\left(
\bar{\mathbf{x}}_{\pi}^k
\right)
\right\|^2
\leq
L^2N\pi_{\max}\Xi^k,
\label{eq:gradient-disagreement}
\end{equation}
where
$\pi_{\max} \triangleq \max_{1\leq i\leq N}\pi_i.$

Taking the Perron-weighted average of the AdamW update and using
$\mathbf{w}^k
=
\bar{\mathbf{x}}_{\pi}^k+\mathbf{r}_x^k$
yields the exact recursion
\begin{equation}
\bar{\mathbf{x}}_{\pi}^{k+1}
=
\bar{\mathbf{x}}_{\pi}^k
-
\alpha
\left[
\nabla F_{\pi}
\left(
\bar{\mathbf{x}}_{\pi}^k
\right)
+
\boldsymbol{\xi}^k
\right],
\label{eq:inexact-GD}
\end{equation}
where
$\boldsymbol{\xi}^k = 
\mathbf{a}^k + \mathbf{t}^k + \bar{\mathbf{g}}_{\pi}^k
- \nabla F_{\pi} \left(\bar{\mathbf{x}}_{\pi}^k\right) + \lambda\mathbf{w}^k
- \frac{1}{\alpha}\mathbf{r}_x^k.$

Define
$\mathcal{E}_k\triangleq \mathbb{E}\|\boldsymbol{\xi}^k\|^2,$ and $\mathcal{W}^k \triangleq \mathbb{E} \|\mathbf{w}^k\|^2.$
Combining \eqref{eq:Rx-bound},
\eqref{eq:A-bound}, and
\eqref{eq:gradient-disagreement} gives
\begin{equation}
\begin{aligned}
\mathcal{E}_k
\leq{}&
5\mathcal{A}^k
+
5\mathcal{T}^k
+
5\lambda^2\mathcal{W}^k
\\
&+
\left(
5L^2N\pi_{\max}
+
\frac{10N\kappa_B}{\alpha^2}
\right)\Xi^k
+
\frac{10N\kappa_A}{\alpha^2}
\mathcal{C}_x^k.
\end{aligned}
\label{eq:E-bound}
\end{equation}

\begin{theorem}[Convergence for smooth convex objectives]
\label{thm:convex}
Suppose Assumptions~\ref{ass:objective}--\ref{ass:bounded} hold and
$0<\alpha\leq1/(2L)$.
Define
\begin{equation}
\widetilde{\mathbf{x}}_{\pi}^{K}
\triangleq
\frac{1}{K}
\sum_{k=0}^{K-1}
\bar{\mathbf{x}}_{\pi}^k,
\qquad
\overline{\mathcal{E}}_{K}
\triangleq
\frac{1}{K}
\sum_{k=0}^{K-1}
\mathcal{E}_k.
\end{equation}
Then,
\begin{equation}
\begin{aligned}
&
\mathbb{E}
\left[
F_{\pi}
\left(
\widetilde{\mathbf{x}}_{\pi}^{K}
\right)
-
F_{\pi}^{\star}
\right]
\\
&\quad\leq
\frac{
\mathbb{E}
\left\|
\bar{\mathbf{x}}_{\pi}^{0}
-
\mathbf{x}_{\pi}^{\star}
\right\|^2
}{
2\alpha K
}
+
R_{\star}
\sqrt{\overline{\mathcal{E}}_{K}}
+
\frac{\alpha(1+2L\alpha)}{2}
\overline{\mathcal{E}}_{K},
\end{aligned}
\label{eq:convex-result}
\end{equation}
where Assumption~\ref{ass:bounded} ensures
$\mathbb{E}
\left\|
\bar{\mathbf{x}}_{\pi}^{k} - \mathbf{x}_{\pi}^{\star} \right\|^2 \leq R_{\star}^2,$
with
$R_{\star}^2 \triangleq 2N\pi_{\max}R_x^2 + 2 \left\| \mathbf{x}_{\pi}^{\star}
\right\|^2.$
Thus, when the perturbations vanish, the standard
$\mathcal{O}(1/K)$ convex rate is recovered.
\end{theorem}

\begin{proof}
Let
$\mathbf{x}^k=\bar{\mathbf{x}}_{\pi}^k$,
$\mathbf{x}^{\star}=\mathbf{x}_{\pi}^{\star}$, and
$\mathbf{g}^k=\nabla F_{\pi}(\mathbf{x}^k)$.
Using \eqref{eq:inexact-GD}, we have
\begin{equation}
\begin{aligned}
\left\|
\mathbf{x}^{k+1}
-
\mathbf{x}^{\star}
\right\|^2
={}&
\left\|
\mathbf{x}^{k}
-
\mathbf{x}^{\star}
\right\|^2
-
2\alpha
\left\langle
\mathbf{x}^{k}
-
\mathbf{x}^{\star},
\mathbf{g}^{k}
\right\rangle
\\
&\hspace{-0.8cm}-
2\alpha
\left\langle
\mathbf{x}^{k}
-
\mathbf{x}^{\star},
\boldsymbol{\xi}^{k}
\right\rangle
+
\alpha^2
\left\|
\mathbf{g}^{k}
+
\boldsymbol{\xi}^{k}
\right\|^2.
\end{aligned}
\end{equation}

For a convex $L$-smooth function,
\begin{equation}
\left\langle
\mathbf{g}^k,
\mathbf{x}^k-\mathbf{x}^{\star}
\right\rangle
\geq
F_{\pi}(\mathbf{x}^k)
-
F_{\pi}^{\star}
+
\frac{1}{2L}
\|\mathbf{g}^k\|^2.
\end{equation}
Using Young's inequality and $\alpha\leq1/(2L)$ therefore gives
\begin{equation}
\begin{aligned}
2\alpha
\left[
F_{\pi}(\mathbf{x}^k)
-
F_{\pi}^{\star}
\right]
\leq{}&
\left\|
\mathbf{x}^k-\mathbf{x}^{\star}
\right\|^2
-
\left\|
\mathbf{x}^{k+1}-\mathbf{x}^{\star}
\right\|^2
\\
&\hspace{-2.1cm}+
2\alpha
\left\|
\mathbf{x}^k-\mathbf{x}^{\star}
\right\|
\|\boldsymbol{\xi}^k\|
+
\alpha^2
(1+2L\alpha)
\|\boldsymbol{\xi}^k\|^2.
\end{aligned}
\label{eq:convex-one-step}
\end{equation}

Taking expectations and applying Cauchy--Schwarz gives
\begin{equation}
\mathbb{E}
\left[
\left\|
\mathbf{x}^k-\mathbf{x}^{\star}
\right\|
\|\boldsymbol{\xi}^k\|
\right]
\leq
R_{\star}
\sqrt{\mathcal{E}_k}.
\end{equation}
Summing over $k=0,\ldots,K-1$, using
$\frac{1}{K}\sum_{k=0}^{K-1}\sqrt{\mathcal{E}_k}\leq\sqrt{\overline{\mathcal{E}}_K},$
and applying convexity of $F_{\pi}$ proves \eqref{eq:convex-result}.
\end{proof}

The next result characterizes the stronger regime in which the aggregate objective satisfies gradient dominance.

\begin{corollary}[Linear rate under the PL condition]
\label{cor:PL}
Under the conditions of Theorem~\ref{thm:convex}, assume that for some $\mu_{\rm PL} > 0$,
\begin{equation}
\frac{1}{2}
\|\nabla F_{\pi}(\mathbf{x})\|^2
\geq
\mu_{\rm PL}
\left[
F_{\pi}(\mathbf{x})
-
F_{\pi}^{\star}
\right]
\label{eq:PL}
\end{equation}
over the region visited by the iterates, and let
$0<\alpha\leq1/(4L)$.
Define
$\Delta_k \triangleq \mathbb{E}\left[F_{\pi}\left(\bar{\mathbf{x}}_{\pi}^k
\right) - F_{\pi}^{\star} \right].$
Then,
\begin{equation}
\Delta_{k+1}
\leq
\rho\Delta_k
+
\frac{5\alpha}{4}\mathcal{E}_k,
\qquad
\rho
\triangleq
1-\alpha\mu_{\rm PL}
<1.
\label{eq:PL-recursion}
\end{equation}
Consequently,
\begin{equation}
\Delta_K
\leq
\rho^K\Delta_0
+
\frac{5\alpha}{4}
\sum_{k=0}^{K-1}
\rho^{K-1-k}\mathcal{E}_k.
\label{eq:PL-unrolled}
\end{equation}
If $\mathcal{E}_k\leq\bar{\mathcal{E}}$, then
$\limsup_{K\rightarrow\infty}
\Delta_K \leq \frac{5\bar{\mathcal{E}}}{4\mu_{\rm PL}}.$
If $\mathcal{E}_k\rightarrow0$, exact convergence follows.
If $\mathcal{E}_k=\mathcal{O}(\gamma^k)$ for some $0<\gamma<1$ with $\gamma\neq\rho$, the convergence is linear with rate $\max\{\rho,\gamma\}$.
For $\gamma=\rho$,
$\Delta_K=\mathcal{O}(K\rho^K)$.
\end{corollary}

\begin{proof}
Let
$\mathbf{x}^k=\bar{\mathbf{x}}_{\pi}^k$
and
$\mathbf{g}^k=\nabla F_{\pi}(\mathbf{x}^k)$.
By $L$-smoothness and \eqref{eq:inexact-GD},
\begin{equation}
F_{\pi}(\mathbf{x}^{k+1})
\leq
F_{\pi}(\mathbf{x}^k)
-\alpha\left\langle \mathbf{g}^k,\mathbf{g}^k+\boldsymbol{\xi}^k \right\rangle
+ \frac{L\alpha^2}{2}\left\|\mathbf{g}^k+\boldsymbol{\xi}^k \right\|^2.
\end{equation}
Using
$-\left\langle
\mathbf{g}^k,
\boldsymbol{\xi}^k \right\rangle \leq \frac{1}{4} \|\mathbf{g}^k\|^2 + \|\boldsymbol{\xi}^k\|^2$
and
$\left\| \mathbf{g}^k + \boldsymbol{\xi}^k \right\|^2 \leq 2\|\mathbf{g}^k\|^2 + 2\|\boldsymbol{\xi}^k\|^2,$
the condition $\alpha\leq1/(4L)$ gives
\begin{equation}
F_{\pi}(\mathbf{x}^{k+1})
\leq
F_{\pi}(\mathbf{x}^k)
-
\frac{\alpha}{2}
\|\mathbf{g}^k\|^2
+
\frac{5\alpha}{4}
\|\boldsymbol{\xi}^k\|^2.
\end{equation}
Applying \eqref{eq:PL} and taking expectations yields
\eqref{eq:PL-recursion}.
Unrolling the recursion gives the remaining claims.
\end{proof}

\begin{remark}
Theorem~\ref{thm:convex} applies directly to smooth convex objectives, including the logistic regression setting used in our experiments.
The PL condition is only an additional sufficient condition for the stronger linear-rate result and is not required by the general convex analysis.
Model-stream compression enters through $\mathcal{C}_x^k$, while tracking-stream compression is captured by $\mathcal{T}^k$.
\end{remark}



\section{Experiment Results and Discussion}

We simulate a decentralized wireless network with $N=15$ nodes deployed in a
$2000\times2000~\mathrm{m}^2$ area via Poisson-disk sampling with
$d_{\min}=250~\mathrm{m}$. Directed links are formed within
$r_{\max}=750~\mathrm{m}$. Channel parameters follow \cite{nguyen2019computation},
with default transmit power $P_i=0.2~\mathrm{W}$. 
To evaluate the proposed method in the smooth convex setting of
Theorem~1, we use a logistic regression model with Top-$K$ biased compression.

\begin{figure}[!ht]
    \centering
    \includegraphics[width=0.98\linewidth]{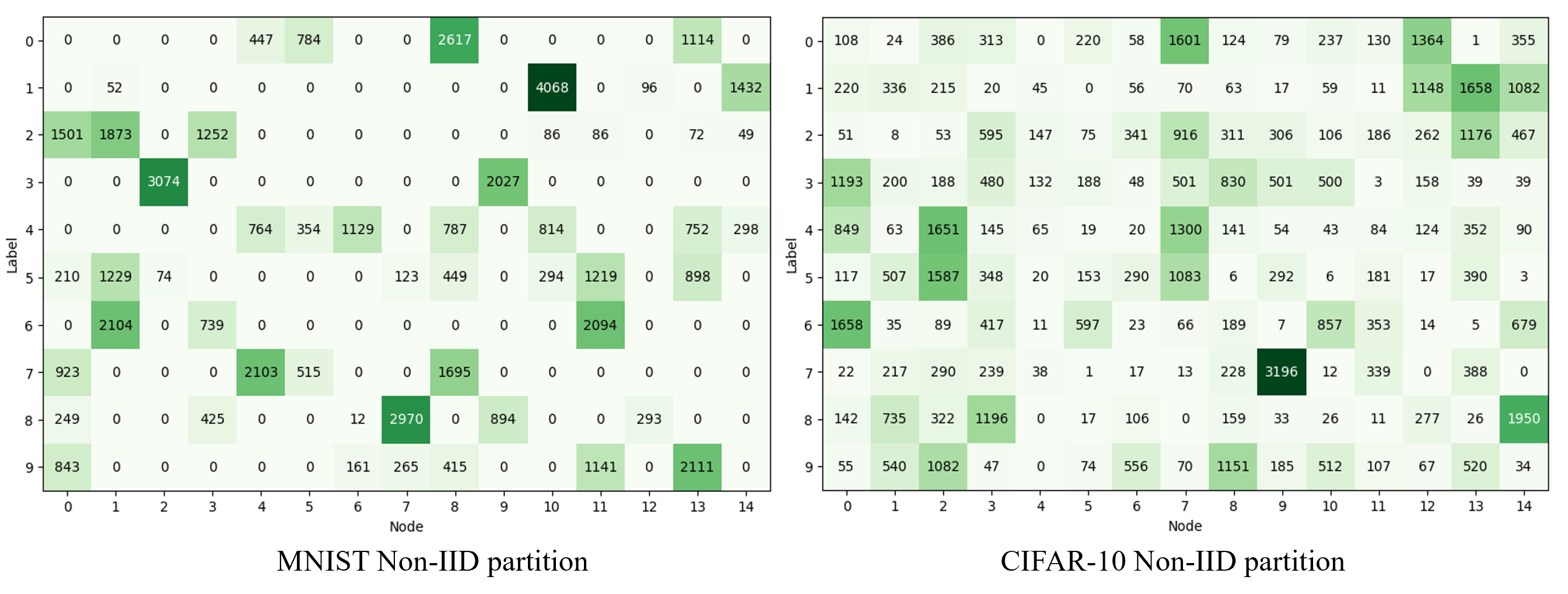}
    \captionsetup{font=footnotesize}
    \caption{Non-IID partition of MNIST and CIFAR10 dataset}
        \vspace{-0.1cm}
    \label{fig:split}
\end{figure}

\begin{figure*}[!ht]
    \centering
    \begin{subfigure}{0.31\textwidth}
        \centering
        \includegraphics[width=\linewidth]{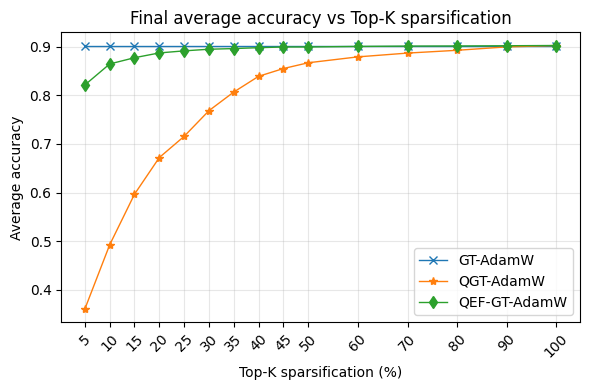}
    \end{subfigure}
    \begin{subfigure}{0.31\textwidth}
        \centering
        \includegraphics[width=\linewidth]{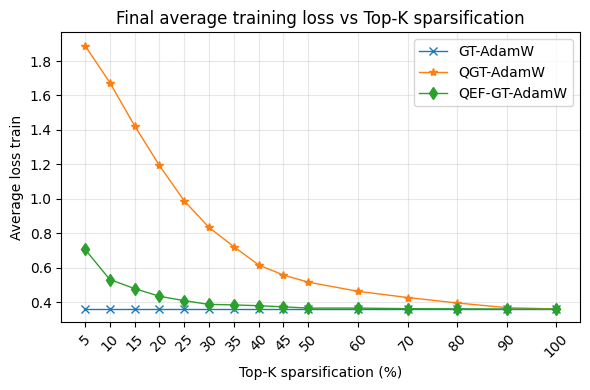}
    \end{subfigure}
    \begin{subfigure}{0.34\textwidth}
        \centering
        \includegraphics[width=\linewidth]{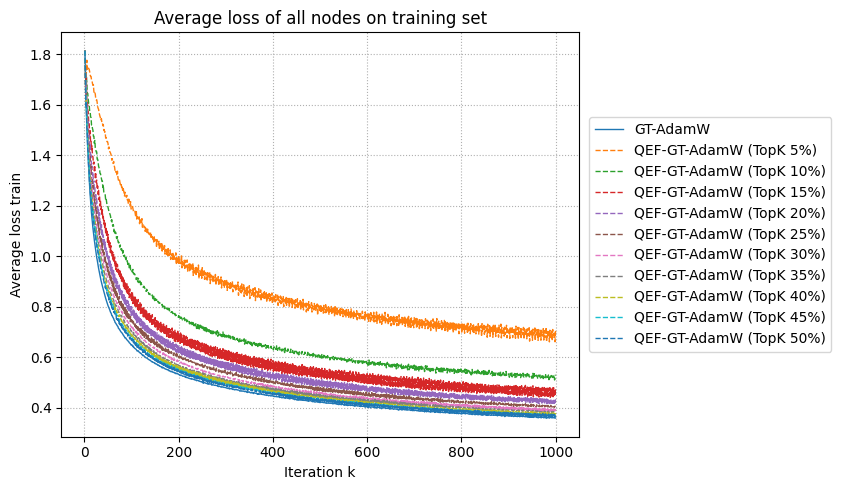}
    \end{subfigure}
    \hfill
    \begin{subfigure}{0.31\textwidth}
        \centering
        \includegraphics[width=\linewidth]{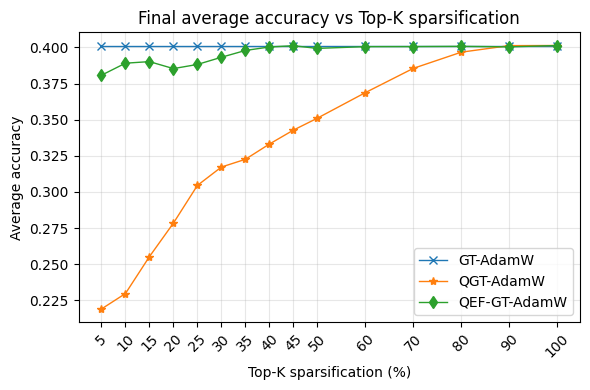}
    \end{subfigure}
    \begin{subfigure}{0.31\textwidth}
        \centering
        \includegraphics[width=\linewidth]{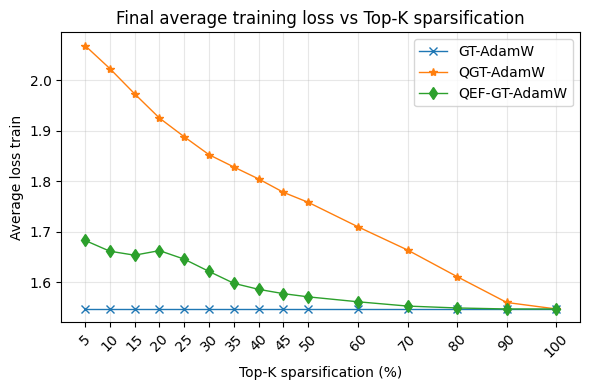}
    \end{subfigure}
    \begin{subfigure}{0.34\textwidth}
        \centering
        \includegraphics[width=\linewidth]{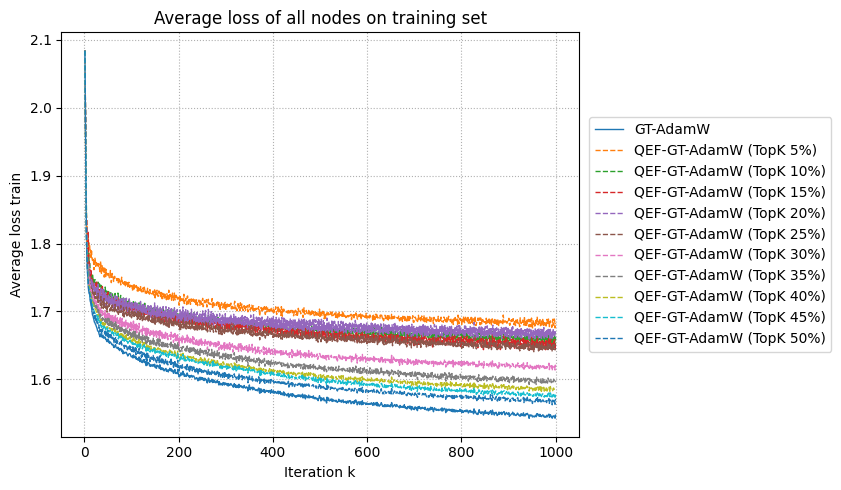}
    \end{subfigure}
    \captionsetup{font=footnotesize}
    \caption{Effect of Top-$K$ density on test accuracy and training loss, with convergence curves for MNIST (first row) and CIFAR-10 (second row).}
    \label{fig:topk_acc}
\end{figure*}

\begin{figure*}[!ht]
    \centering
    \begin{subfigure}{0.3\textwidth}
        \centering
        \includegraphics[width=\linewidth]{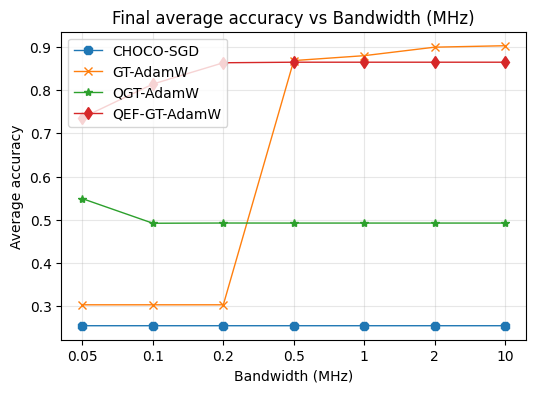}
        \captionsetup{font=footnotesize}
        \caption{Mean accuracy (MNIST)}
    \end{subfigure}
    \begin{subfigure}{0.31\textwidth}
        \centering
        \includegraphics[width=\linewidth]{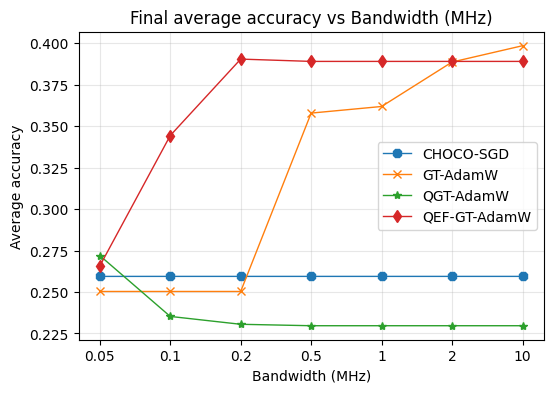}
        \captionsetup{font=footnotesize}
       \caption{Mean accuracy (CIFAR-10)}
        \label{fig:bw_cifar10_acc}
    \end{subfigure}
    \begin{subfigure}{0.31\textwidth}
        \centering
        \includegraphics[width=\linewidth]{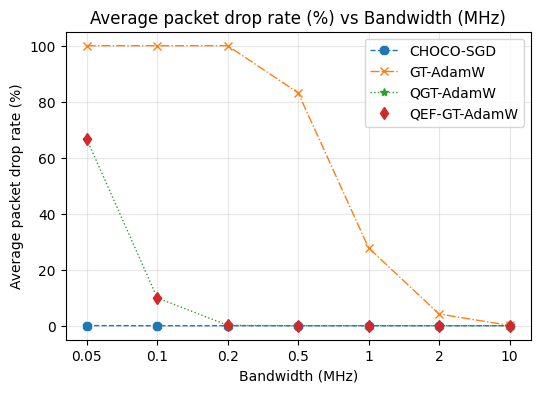}
        \captionsetup{font=footnotesize}
        \caption{Average packet drop rate versus bandwidth}
        \label{fig:bw_cifar10_loss}
    \end{subfigure}
    \captionsetup{font=footnotesize}
    \caption{Final average accuracy and average packet drop rate versus bandwidth (MHz) under wireless deadlines and outage-based drops}
    \label{fig:bw_tune}
\end{figure*}

\begin{figure*}[!ht]
    \centering
    \begin{subfigure}{0.3\textwidth}
        \centering
        \includegraphics[width=\linewidth]{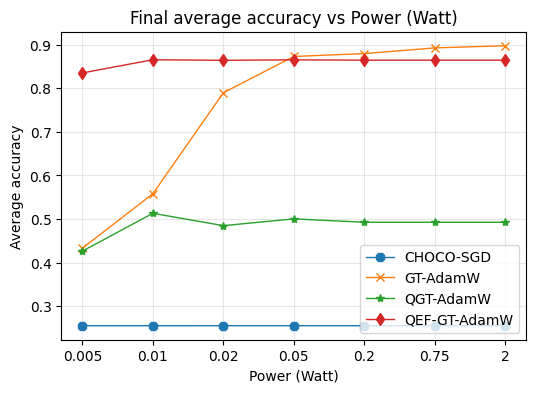}
        \captionsetup{font=footnotesize}
        \caption{Mean accuracy (MNIST)}
    \end{subfigure}
    \begin{subfigure}{0.31\textwidth}
        \centering
        \includegraphics[width=\linewidth]{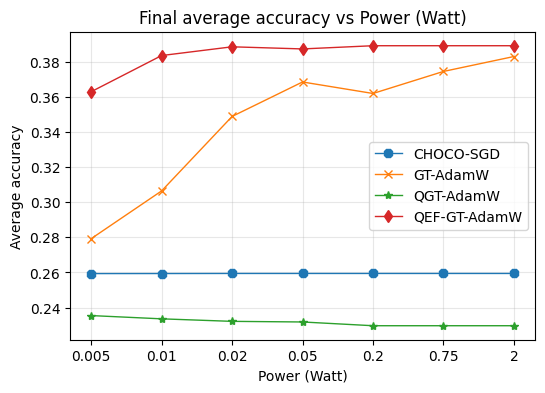}
        \captionsetup{font=footnotesize}
       \caption{Mean accuracy (CIFAR-10)}
        \label{fig:pw_cifar10_acc}
    \end{subfigure}
    \begin{subfigure}{0.31\textwidth}
        \centering
        \includegraphics[width=\linewidth]{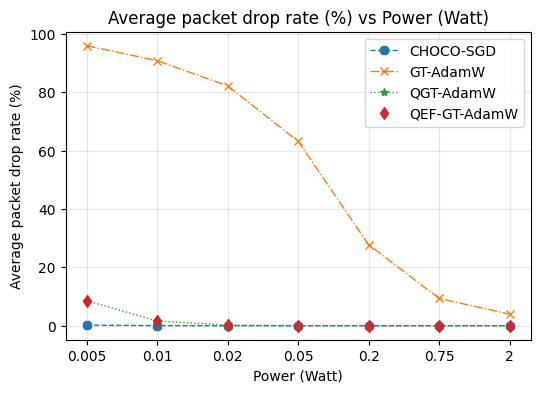}
        \captionsetup{font=footnotesize}
        \caption{Average packet drop rate versus power}
        \label{fig:pw_cifar10_loss}
    \end{subfigure}
    \captionsetup{font=footnotesize}
    \caption{Final average accuracy and average packet drop rate versus transmission power (Watt) under wireless deadlines and outage-based drops.}
    \label{fig:power_tune}
\end{figure*}

\subsection{Datasets and Model}
\label{subsec:datasets_tasks}
We evaluate MNIST~\cite{lecun2002gradient} and CIFAR-10~\cite{krizhevsky2009learning}
using 50,000 training samples and the standard test sets over 1000 epochs.
Non-IID data are split across $N$ local datasets $\{\mathcal{D}_i\}_{i=1}^{N}$,
as shown in Fig.~\ref{fig:split}: MNIST uses a pathological label-skew split
with 2--5 labels per node~\cite{huang2021personalized}, while CIFAR-10 uses
Dirichlet partitioning with concentration $\alpha=0.5$~\cite{gao2022feddc}.

For comparison purposes, we also consider three representative DecL baselines: (1) CHOCO-SGD \cite{koloskova2019decentralized}, a compressed decentralized SGD method with quantized communication; (2) GT-AdamW, which integrates classical GT with AdamW; and (3) QGT-AdamW, a quantized GT with AdamW.
Training hyperparameters are as follows: for AdamW-based methods, the learning rate $\eta$ is set to $0.005$ (MNIST) and $0.001$ (CIFAR-10) with $\beta_m=0.9$ and $\beta_v=0.999$; for CHOCO-SGD, $\eta=0.005$ (MNIST), $\eta=0.001$ (CIFAR-10) and consensus step-size $\gamma = 0.001$ for both datasets. All methods use a weight decay of $0.01$. 
Communication compression uses Top-$K$ sparsification retaining
10\% of the coordinates, achieving 90\% message reduction.
Performance is evaluated using three
metrics: (1) test accuracy, (2) global training loss, and
(3) packet drop rate.

\subsection{Results and Discussion}

Fig.~\ref{fig:topk_acc} illustrates the impact of quantization and EF  on model performance. In this experiment, wireless impairments are disabled, and perfect communication is assumed at every iteration to isolate the effect of \emph{quantization only}. 
Under this controlled setting, GT-AdamW serves as a full-precision reference because it does not incur quantization error.
The results show that \texttt{QEF-GT-AdamW} consistently outperforms the baseline QGT-AdamW across all Top-$K$ sparsifications. Moreover, \texttt{QEF-GT-AdamW} approaches the full-precision reference accuracy at Top-$K$ sparsifications of approximately $30\%$ for MNIST and $35\%$ for CIFAR-10. In contrast, the baseline QGT-AdamW requires a substantially larger Top-$K$ sparsification (around $90\%$) to achieve comparable accuracy.
We further visualize the average training loss across all nodes for different Top-$K$ sparsifications. For both datasets, \texttt{QEF-GT-AdamW} yields uniformly lower loss than QGT-AdamW at the same Top-$K$ sparsification, confirming that EF effectively compensates quantization bias and improves optimization quality. 

Figs.~\ref{fig:bw_tune} and~\ref{fig:power_tune} further explain the accuracy trends by reporting the \emph{average packet drop rate during training} under bandwidth-limited and power-limited WCom, respectively, with the Top-$K$ sparsification fixed at $10\%$. As can be seen, \textsc{CHOCO-SGD} consistently exhibits an almost zero packet drop rate over all settings because it exchanges only the model variable, resulting in a much smaller packet size. However, despite this highly reliable communication behavior, its learning accuracy remains significantly lower, since it does not explicitly mitigate client drift under non-IID data.
By contrast, GT-based methods must exchange both the model and the gradient-tracking variables, which substantially enlarges the payload and makes them much more vulnerable to deadline violations and wireless outages. This effect is most severe for \texttt{GT-AdamW}, which does not use quantization. As shown in Fig.~\ref{fig:bw_tune}, its average packet drop rate is approximately $100\%$ at $0.05$, $0.1$, and $0.2$ MHz, remains very high at about $83\%$ even at $0.5$ MHz, and is still around $28\%$ at $1$ MHz. Only when the bandwidth increases to $2$ MHz and above does the drop rate become small (about $4\%$ at $2$ MHz and nearly $0\%$ at $10$ MHz). A similar trend is observed in Fig.~\ref{fig:power_tune}: \texttt{GT-AdamW} experiences about $96\%$, $91\%$, $82\%$, and $63\%$ packet drops at $0.005$, $0.01$, $0.02$, and $0.05$ W, respectively, and its drop rate decreases gradually only at higher power levels (about $28\%$ at $0.2$ W, $9\%$ at $0.75$ W, and $4\%$ at $2$ W). These statistics explain the strong performance degradation of unquantized GT-AdamW in bandwidth- or power-limited regimes.

In contrast, the quantized GT variants, \texttt{QGT-AdamW} and \texttt{QEF-GT-AdamW}, achieve dramatically lower packet drop rates because both the model and tracking streams are compressed before transmission. In Fig.~\ref{fig:bw_tune}, their drop rate is about $67\%$ at $0.05$ MHz, decreases to roughly $10\%$ at $0.1$ MHz, and becomes essentially zero from $0.2$ MHz onward. In Fig.~\ref{fig:power_tune}, their drop rate is only about $8\%$ at $0.005$ W, around $1\%$ at $0.01$ W, and nearly zero for $0.02$ W and above. Since \texttt{QGT-AdamW} and \texttt{QEF-GT-AdamW} use the same Top-$K$ sparsification, they exhibit nearly identical drop-rate curves, which is expected because their transmitted packet sizes are almost the same. Therefore, the accuracy gain of \texttt{QEF-GT-AdamW} over \texttt{QGT-AdamW} does not come from a lower transmission failure probability, but from the optimization benefit of EF in compensating compression distortion across rounds. Overall, these results confirm that \texttt{QEF-GT-AdamW} achieves a more favorable trade-off among communication efficiency, robustness to wireless impairments, and learning effectiveness under heterogeneous data.


\section{Conclusion}


In this paper, we propose \texttt{QEF-GT-AdamW}, a resilient and communication-efficient DecL algorithm for unreliable wireless edge networks. It combines AdamW, gradient tracking, and dual biased quantization-based EF to minimize payloads and stabilize optimization against non-IID data. Experiments on MNIST and CIFAR-10 show that QEF-GT-AdamW outperforms representative baselines in convergence speed and robustness, particularly in resource-limited environments.

\ifCLASSOPTIONcaptionsoff
  \newpage
\fi

\section*{Acknowledgment}
\small{
This work was supported by the Luxembourg National Research Fund (FNR) through the CHIST-ERA SHIELD project (Grant INTER/CHIST24/19023763/SHIELD) and the FNR AFR program via the FULFIL project (Grant 2000141).
}

\bibliographystyle{IEEEtran}
\bibliography{IEEEabrv,refs_globecom.bib} 

%

\begin{IEEEbiography}{Michael Shell}
Biography text here.
\end{IEEEbiography}

\begin{IEEEbiographynophoto}{John Doe}
Biography text here.
\end{IEEEbiographynophoto}


\begin{IEEEbiographynophoto}{Jane Doe}
Biography text here.
\end{IEEEbiographynophoto}



\end{document}